\documentclass[letterpaper]{article} 
\usepackage{multirow}
\usepackage{aaai2026}  
\usepackage{times}  
\usepackage{helvet}  
\usepackage{courier}  
\usepackage[hyphens]{url}  
\usepackage{graphicx} 
\usepackage{natbib}  
\usepackage{caption} 
\usepackage{graphicx} 
\usepackage{booktabs}
\usepackage{amsthm}
\usepackage{algorithm}
\usepackage{algorithmic}
\usepackage{algorithm}
\usepackage{algorithmic}
\theoremstyle{plain}
\newtheorem{theorem}{Theorem}[section]
\newtheorem{proposition}[theorem]{Proposition}
\newtheorem{lemma}[theorem]{Lemma}
\newtheorem{corollary}[theorem]{Corollary}
\theoremstyle{definition}
\newtheorem{definition}[theorem]{Definition}

\theoremstyle{remark}

\usepackage{amssymb}  
\usepackage{amsmath}  
\usepackage{booktabs}

\usepackage{amssymb}  

\usepackage{amsmath}
\usepackage{cleveref}

\usepackage{newfloat}
\usepackage{listings}
\DeclareCaptionStyle{ruled}{labelfont=normalfont,labelsep=colon,strut=off} 
\floatstyle{ruled}
\newfloat{listing}{tb}{lst}{}
\floatname{listing}{Listing}
\title{Learning Topology-Driven Multi-Subspace Fusion \\for Grassmannian Deep Networks}
\author{
    Xuan Yu\textsuperscript{\rm },
    Tianyang Xu\textsuperscript{\rm }\thanks{Corresponding author.}
}
\affiliations{
    \textsuperscript{\rm }School of Artificial Intelligence and Computer Science, Jiangnan University\\
    Wuxi, Jiangsu 214122, China\\

    xuan.yu@stu.jiangnan.edu.cn, tianyang.xu@jiangnan.edu.cn
}

\usepackage{bibentry}

\begin{document}

\maketitle

\begin{abstract}
Grassmannian manifold offers a powerful carrier for geometric representation learning by modelling high-dimensional data as low-dimensional subspaces. 
However, existing approaches predominantly rely on static single-subspace representations, neglecting the dynamic interplay between multiple subspaces critical for capturing complex geometric structures. 
To address this limitation, we propose a topology-driven multi-subspace fusion framework that enables adaptive subspace collaboration on the Grassmannian. 
Our solution introduces two key innovations: (1) Inspired by the Kolmogorov-Arnold representation theorem, an adaptive multi-subspace modelling mechanism is proposed that dynamically selects and weights task-relevant subspaces via topological convergence analysis, and (2) a multi-subspace interaction block that fuses heterogeneous geometric representations through Fréchet mean optimisation on the manifold. 
Theoretically, we establish the convergence guarantees of adaptive subspaces under a projection metric topology, ensuring stable gradient-based optimisation.
Practically, we integrate Riemannian batch normalisation and mutual information regularisation to enhance discriminability and robustness. 
Extensive experiments on 3D action recognition (HDM05, FPHA), EEG classification (MAMEM-SSVEPII), and graph tasks demonstrate state-of-the-art performance. 
Our work not only advances geometric deep learning but also successfully adapts the proven multi-channel interaction philosophy of Euclidean networks to non-Euclidean domains, achieving superior discriminability and interpretability.
\end{abstract}

\begin{links}
    \link{Code}{https://github.com/Xua-Yu/GMSF-Net}
    \link{Extended version}{https://arxiv.org/abs/2511.08628}
\end{links}

\section{Introduction}

\label{introduction}
In recent years, unitary subspace modelling on the Grassmannian has achieved great success in tasks which require extended discriminative capacity, such as activity recognition~\cite{cherian2017generalized}, emotion recognition~\cite{liu2014combining,wang2023u}, face verification~\cite{huang2015projection,chen2021hybrid}, and classification of time-series data in brain-computer interfaces~\cite{gao2022domain}. 
In the above tasks, Grassmannian typically model the input data with a single and static orthogonal subspace~\cite{huang2018building,wang2024grassmannian,xu2023learning}. 
However, a fixed single-subspace assumption fails to capture local geometric variations and multi-modal distribution, limiting the expressiveness and adaptability of the model in structurally diverse or heterogeneous task scenarios~\cite{hamm2008grassmann}.
To overcome this issue, we introduce an adaptive multi-subspace representation that dynamically adjusts multiple subspaces to accurately capture task-specific characteristics and the harmonious distribution among subspaces.
The convergence of obtaining adaptive subspaces is analysed from a topological perspective~\cite{munkres2000topology}, satisfying specific demands across different tasks.

In deep neural networks, different vectors in the Euclidean space reflect their specific feature identities. 
However, the existing Euclidean metric~\cite{olver2006applied} driven by the inner product, measures the semantic difference between two vectors, resulting in a distorted relevance in task-specific decisions (as shown in Figure \ref{figure1} (a) and (d)). 
Based on non-Euclidean geometry, this distorted relevance can be effectively addressed on Riemannian manifolds~\cite{chen2023riemannian}. 
Intuitively, Riemannian manifolds provide compact spaces embedded in the dense Euclidean space, where the accessible path can preserve the geometric consistency~\cite{boumal2023introduction}. 
As a typical Riemannian manifold, the Grassmannian~\cite{bendokat2024grassmann} can model high-dimensional data as linear subspaces to effectively address semantic distortion. 
Despite this merit, existing solutions~\cite{huang2018building,wang2024grassmannian} are primarily driven by a single subspace, lacking consideration of the collaborative interaction of multi-subspaces, which hinders the characterisation of subspace geometry diversity in terms of extended semantics (as shown in Figure \ref{figure1} (b) and (e)). 
\begin{figure*}
\centering
\includegraphics[width=\linewidth]{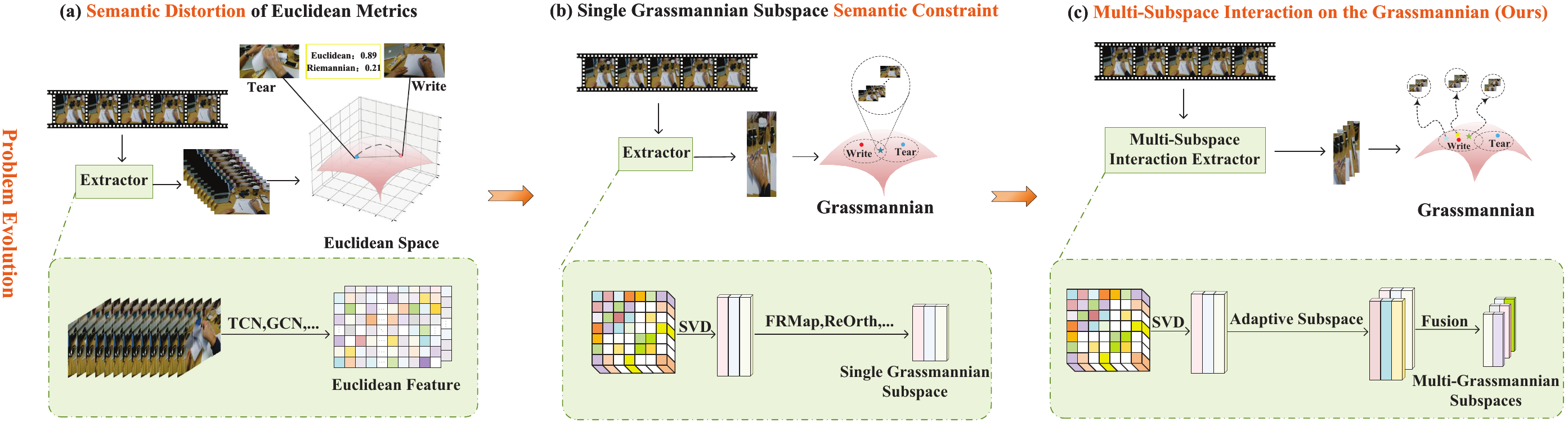}
\caption{Illustration of the problem evolution. (a) shows semantic distortion of features in Euclidean space. (b) reveals semantic constraint using a single Grassmannian subspace. (c) presents the proposed framework incorporating multi-Grassmannian subspace fusion.}
\label{figure1}
\end{figure*}
In the Euclidean space, the promising progress of deep learning in the recent decade is attributed to multi-channel interactions and non-linear activations~\cite{krizhevsky2012imagenet}. 
In particular, the seminal work LeNet-5~\cite{lecun2002gradient} extracts multi-channel feature maps using stacked convolution blocks, thus obtaining extended representative capacity via channel interaction. 
Admitting the natural representation power of linear subspace~\cite{xu2018non}, the potential of multi-subspace interactions has been overlooked in the community. 
To reveal the power of multi-subspace interaction for the Grassmannian, there are two major issues, \textit{i.e.}, \textit{how to perform subspace interaction and how to formulate its deep architecture.}

For the first issue, we design a multi-subspace interaction block in a geometric manner, \textit{i.e.}, several intermediate representations are generated by a unified mapping and then fused through alignment, facilitating effective interaction among different Grassmannian subspaces. 
For the second issue, we leverage an intrinsic geometric approach to ensure that the fusion process occurs at accurate positions on the manifold~\cite{tao2024alignmif}, which enables the design of our multi-subspace interaction block in a stackable form, allowing for flexible capacity expansion~\cite{song2025refinefuse}.
To the best of our knowledge, this is the first work that introduces deep interactions of Grassmannian subspaces in Riemannian neural networks, as shown in Figure \ref{figure1} (c) and (f), these deep interactions capture the dynamic relationships between multiple subspaces, effectively modelling their geometric structures, semantic features, and mutual influences.
Euclidean deep networks naturally implement backpropagation, given that the gradient superposition property satisfies the axioms of linear spaces~\cite{fei2025survey}. 
However, parameter optimisation on a Riemannian manifold must be strictly restricted to its tangent space ~\cite{smith2014optimization}. 
To this end, we develop a topology-driven framework for adaptive multi-subspace construction on the Grassmannian.
In summary, our main contributions are the following: 
\begin{itemize}
    \item We propose an adaptive multi-subspace modelling approach tailored to diverse recognition tasks.
    \item We propose a novel Grassmannian Multi-Subspace Fusion network (GMSF-Net) to fuse heterogeneous subspace representations.
    \item We propose a topology-driven framework with theoretical guarantees for adaptive Grassmannian subspaces construction.
\end{itemize}

\section{Related Work}
The Grassmannian is the set of all linear subspaces of a fixed dimension, showing notable algebraic and geometric merits, reflecting the stability and invariance of subspace geometric structures~\cite{souza2020interface}. In related research, deep learning applications on the Grassmannian have evolved from shallow methods to deep models.

In the early research stage, shallow Grassmannian learning~\cite{zhang2018grassmannian} is achieved by projecting subspaces onto the manifold, such as discriminant analysis~\cite{baudat2000generalized}, high-dimensional data clustering~\cite{von2007tutorial}, and low-rank matrix completion~\cite{dai2012geometric}. 
However, shallow learning on the manifold cannot mine deep geometric characteristics from the data, thus only suitable for simple tasks. 
The development of GrNet breaks this limitation by formulating a deep network version of the Grassmannian~\cite{huang2018building}. 
In particular, GrNet projects data onto the Grassmannian and follows the manifold's geometric properties to design network layers such as FRMap, OrthMap, and ProjMap. 
A brief summary of GrNet is provided in Appendix B. 
Typically, GrNet considers the manifold's constraints and the low-dimensional subspace representation of the data, thereby enhancing the model capacity to learn delicate geometric structures and effectively solve nonlinear relationships and high-dimensional problems. 

In recent studies, equipped with the self-attention mechanism, models can focus on the global relevance of the input and build long-range dependencies among them, greatly enhancing the model capacity~\cite{vaswani2017attention}. 
To this end, to explore the potential of self-attention on the Grassmannian, GDLNet~\cite{wang2024grassmannian} integrates the mechanism to build dependencies between different subspaces. 
It helps the model effectively focus more on important subspace information. 
However, GrNet and GDLNet typically adopt a single and static subspace to model the original global features, which makes it difficult to extract diverse geometric subspaces, thus delivering internal disturbances and weaker discriminative power. 
To address this issue, we introduce an adaptive multi-subspace construction mechanism to capture coherent geometric features. 
Additionally, we design a multi-subspace interaction block to fuse different subspaces, facilitating intra-subspace transformation and inter-subspace complement.

\section{Preliminaries}
This section provides a brief overview of the basic geometries of Grassmannian and Topology. 
More detailed review and relevant notations can be found in Appendix A.

\subsection{Grassmannian}
\textbf{Grassmannian geometries:} The Grassmannian \( \mathcal{G}(n, p) \) is the set of all \( p \)-dimensional linear subspaces of \( \mathbb{R}^n \). 
Each point represents a \( p \)-dimensional subspace, typically represented by an orthonormal basis matrix \( U \). 
The tangent space, exponential map, and logarithmic map are key tools to understand its geometric structure, and these concepts will serve as prerequisites for subsequent operations (a detailed explanation is provided in Appendix A.2).

\noindent \textbf{Grassmannian Metric:}
Points on the Grassmannian represent different subspaces, and their distances are typically measured by metrics such as the Projection metric~\cite{huang2015projection} or Principal Angles metric~\cite{qiu2005unitarily}. 
Considering computational efficiency, this paper adopts the Projection metrics to analyse points on the manifold. 
For two subspaces \(X_1\) and \(X_2\), the Projection metric is defined as:
\begin{equation}
\mathit{d_p}(\mathit{X}_1, \mathit{X}_2) = 2^{-1/2} \|\mathit{X}_1 \mathit{X}_1^T - \mathit{X}_2 \mathit{X}_2^T\|_F ,
\label{project}
\end{equation}
where \( \|\cdot\|_F \) is the Frobenius norm~\cite{harandi2013dictionary}. 
This metric serves as an effective tool for quantifying the distance between subspaces and enabling geometric analysis on the Grassmannian~\cite{fei2025survey}.

\subsection{Topology}
Topological spaces underpin Riemannian geometry by abstracting spatial properties beyond specific shapes~\cite{lee2000introduction}. 
On the Grassmannian, each subspace inherently constitutes a topological space. 
This topological framework rigorously characterises subspace structure and metric, ensuring the critical convergence property for sequences of subspaces. 
This theoretical guarantee of convergence motivates the design of adaptive multi-subspace modelling, ensuring stable convergence towards limiting subspaces during optimisation and enabling efficient geometric feature extraction and fusion. 
This section provides a brief introduction to the topological structure, with detailed definitions and proofs presented in Appendices A.3 and F.
\begin{figure*}
\centering
\includegraphics[width=\linewidth]{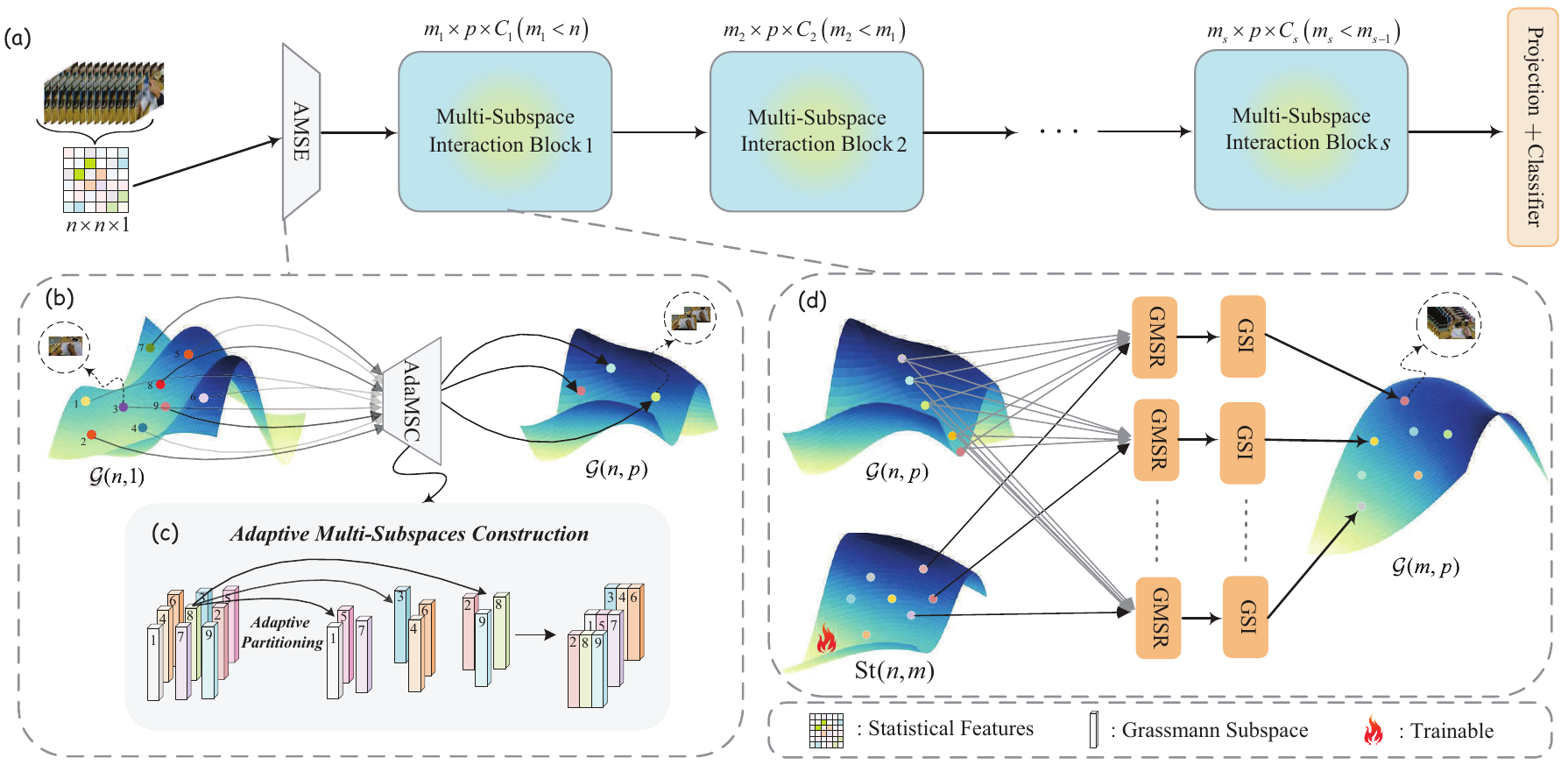}
\caption{\textit{Overview of GMSF-Net.} (a) Architecture of GMSF-Net. (b) Structure of the Adaptive Multi-Subspace Encoder (AMSE). (c) \textit{Fine-Grained Design.} The proposed efficient Adaptive Multi-Subspace Construction (AdaMSC). (d) \textit{Subspace Interaction Design.} The proposed discriminative Multi-Subspace Interaction Block.}
\label{figure2}
\end{figure*}

\section{Approach}
\label{Approach}
In this section, we will systematically introduce our method. 
We first introduce the adaptive multi-subspace encoder on the Grassmannian, analysing its design principles in depth and proving its convergence based on topological theory. 
Subsequently, we explain the design of the multi-subspace interaction block based on a geometric approach in detail. 
Figure \ref{figure2} (a) presents an overview of GMSF-Net.

\subsection{Adaptive Multi-Subspace Encoder (AMSE)}
This subsection introduces the Adaptive Multi-Subspace Encoder (AMSE) on the Grassmannian, which adaptively constructs multiple coherent subspaces to preserve the diverse geometric characteristics of data, as illustrated in Figure \ref{figure2} (b). 
This subsection introduces the construction of adaptive multiple subspaces and the analysis of their convergence in detail.

\subsubsection{Adaptive Multi-Subspace Construction (AdaMSC)}
\label{AdaMSC}
The construction of adaptive multi-subspaces is the core of AMSE, as shown in Figure \ref{figure2} (c), which builds multiple stable and harmonious low-dimensional subspace representations on the high-dimensional Grassmannian space \( \mathcal{G}(n,k)\), and adaptively adjusts these subspaces based on different tasks. 
In particular, we first extract frame-level features and compute a covariance matrix $X \in \mathbb{R}^{n \times n}$ to model the statistical dependencies of features along the temporal dimension~\cite{nguyen2019neural}. 
Subsequently, we apply the Schmidt orthogonalisation process to map $X$ into a set of low-dimensional orthogonal subspaces~\cite{yaghooti2024gram}:
\begin{equation}
\mathcal{S} = \{ S_1, S_2, \dots, S_k \}, \quad S_i^\top S_j = 0 \quad (\forall\, i \neq j),
\end{equation}
where each subspace \( S_j \in \mathcal{G}(n,1) \). 
This design aims to decouple statistical information from physical information in the features, transforming them into a set of localised geometric encodings that facilitate subsequent subspace modeling~\cite{edelman1998geometry}.

We define the index set of subspaces as \(\mathcal{I} = \{i_1, \dots, i_k\}\), where each index \(i_j\) corresponds to a subspace \(S_{j} \in \mathcal{G}(n,1)\). 
For each new subspace \(S'_{m'}\), we initialise a set of learnable weight parameters \(\mathcal{W}^{(m')} = \{w_1^{(m')}, \dots, w_k^{(m')}\}\), which are normalised via the Softmax function to obtain the weight distribution $\widetilde{\mathcal{W}}^{(m')} = \{\widetilde{w}_1^{(m')}, \ldots, \widetilde{w}_k^{(m')}\}$.  
Each $\widetilde{w}_j^{(m')}$ reflects the relative importance of the atomic subspace $S_j$ during the construction of $S'_{m}$. 
We then select the top $p$ indices with the highest values in $\widetilde{\mathcal{W}}^{(m')}$, obtaining the index set 
$\{i_{j_1}, i_{j_2}, \ldots, i_{j_p}\}$ and the corresponding weights 
$\{\widetilde{w}_{j_1}^{(m')}, \ldots, \widetilde{w}_{j_p}^{(m')}\}$. 
The \(m'\)-th new subspace is constructed from these weighted subspaces, with the construction defined as above based on Theorem F.1 in the Appendix:
\begin{equation}
S'_{m'} =[ \widetilde{w}_{j_1}^{(m')} S_{j_1},\ \widetilde{w}_{j_2}^{(m')} S_{j_2},\ \dots,\ \widetilde{w}_{j_p}^{(m')} S_{j_p}].
\end{equation}

In this way, the new subspace $S'_{m'}\in \mathcal{G}(n,p)$ is composed of distinct key atomics which enable efficient modelling and adaptive adjustment of diverse features and lay a solid foundation for subsequent multi-subspace fusion. 
The detailed process is shown in~\cref{alg:weighted}.

\begin{algorithm}[t]
   \caption{Adaptive Multi-Subspace Construction}
   \label{alg:weighted}
\begin{algorithmic}[1]
   \STATE \textbf{Input:}
      Low-dimensional subspaces $\{ S_i \}_{i=1}^k$, $S_i \in \mathcal{G}(n,1)$,
      subspace dimension $p$, number of output subspaces $m'$
   \STATE \textbf{Parameter:} Importance matrix $\mathcal{W} \in \mathbb{R}^{m' \times k}$
   \STATE \textbf{Output:} Weighted multi-subspaces $\{ S'_i \}_{i=1}^{m'}$
   
   \STATE $\mathcal{W'} \leftarrow \text{Softmax}(\mathcal{W}, \text{dim}=1)$ \textit{\# row-wise softmax}
   \FOR{$i = 1$ \textbf{to} $m'$}
     \STATE Sort indices: $I_i \gets argsort_{descending}(\mathcal{W'}[i,:])$
     \STATE Select top-$p$ components: $\widetilde{S}_i \gets S_i[:, I_i[1:p]]$ 
     \STATE Get corresponding weights: $\mathbf{w}_i \gets \mathcal{W'}[i, I_i[1:p]]$
     \STATE $T_i \gets \widetilde{S}_i \odot (\mathbf{1}_n \cdot \mathbf{w}_i^\top)$ \COMMENT{$\mathbf{1}_n$: all-ones vector}
   \ENDFOR
   \STATE Concatenate subspaces: $\{ S'_i \}_{i=1}^{m'} \gets \bigcup_{k=1}^{m'} \{ T_k \}$ 
\end{algorithmic}
\end{algorithm}

Next, we introduce the mutual information function \( f(S') \) to measure the complementarity between different new subspaces \( S_i' \) and \( S_j' \)~\cite{liu2009feature}. This function aims to enhance the discriminative ability of the overall representation by maximising the information complementarity between subspaces:
\begin{equation}
f(S') = I(S_i', S_j'), \quad i \neq j, \quad i,j \in \{1, \dots, m'\}.
\end{equation}

The specific formula for the mutual information \( I(S_i', S_j') \) is:
\begin{equation}
I(S_i', S_j') = \int \int p(S_i', S_j') \log \frac{p(S_i', S_j')}{p(S_i') p(S_j')} \, dS_i' \, dS_j'.
\end{equation}

In practical modelling, since it is difficult to directly obtain the probability distributions \( p(S_i') \) and \( p(S_j') \) for the subspaces, we typically use embedding or feature-based estimation methods to approximate the mutual information (for detailed information see the \textit{Optimisation Strategies})~\cite{maeda2025estimation}.

\subsubsection{Analysis of Adaptive Subspace Convergence}
We further analyse multi-subspace construction on the Grassmannian from a topological perspective to ensure model convergence. By introducing a metric topology and projection metric, we ensure the subspaces converge iteratively. In the metric topology, the distance on the Grassmannian is defined by the projection metric, as shown in Equation (1). This metric satisfies non-negativity, symmetry, and the triangle inequality, ensuring that the Grassmannian is a metric topology. The proof is in Appendix F.2.

In this metric topology, the convergence of a sequence of subspaces is naturally defined as detailed in Proposition F.3 in the Appendix. During training~\cite{ionescu2015training}, the subspace is iteratively updated through Riemannian gradient descent, with the update rule as follows:
\begin{equation}
S'(t+1) = \text{Exp}_{S'(t)}\left(-\eta \, \text{Proj}_{T_{S'(t)}} \nabla L(S')\right),
\end{equation}
where \(\text{Exp}_{S'(t)}\) is the exponential map, \(\text{Proj}_{T_{S'(t)}}\) is the tangent space projection, \(\nabla L(S')\) is the gradient,and the step size \(\eta\) controls the learning rate. The proof is in Appendices F.3 to F.5. Due to the continuity and differentiability of the loss function with a lower bound, along with the completeness of the manifold and the continuity of the exponential map~\cite{eschenauer2001topology}, this subspace eventually satisfies (see Appendix F.6):
\begin{equation}
d(S'(t), S^*) \to 0 \quad \text{as} \quad t \to \infty .
\end{equation}

This means that in the topology induced by the projection metric, the sequence converges to the stable subspace $S^*$, ensuring the convergence and consistency of the model's geometric modelling process.

\subsection{Multi-Subspace Interaction Block}
\label{Multi-Subspace}
The multi-subspace interaction block fuses multiple Grassmannian subspaces into a single integrated subspace, as shown in Figure \ref{figure2}, effectively capturing geometric features of data and enhancing the ability to represent complex geometric cues. 
This block consists of two main components: Grassmannian Multi-Representations and Subspace Interaction.

This work introduces a semantic perspective for multi-subspace interaction design, addressing the often-overlooked intuition behind such interactions. 
We posit that distinct subspaces, governed by different geometric frameworks, capture complementary task-relevant information through diverse local geometric features.
Their interaction transcends mere feature fusion, embodying semantic co-aggregation within the manifold. 
By incorporating alignment constraints and weighted aggregation, our model dynamically integrates multi-subspace representations, enhancing both discriminability and geometric diversity.

\subsubsection{Grassmannian Multi-Subspace Representations (GMSR)}
\label{GMR}
Grassmannian Multi-Subspace Representations (GMSR) generates subspace representations with distinct geometric features under a unified geometric framework by applying the same learnable matrix to a set of input subspaces $\{S'_{m'}\}$, producing a set of output representations $\{X_{GMSR}^{c,m'}\}$. 
The formulation is as follows:
\begin{equation}
X_{GMSR}^{c,m'} = f_{GMSR}^{(c)}(S'_{m'}, W_c) = W_c^{T} S'_{m'},
\end{equation}
where \(S'_{m'} \in \mathcal{G}(n, p)\) is the $m$-th output subspace from the AMSE module, \( W_{c} \in \mathbb{R}^{n \times m}, (m < n) \) is the mapping matrix that is basically required to be a full rank matrix~\cite{huang2017riemannian}, \( X_{GMSR}^{c,m'} \in \mathcal{G}(m, p) \) is the new Grassmannian representations. 
Generally, \( W_{c}^{T} S' \) is not an orthonormal basis matrix,  we employ QR or SVD decomposition to preserve the intrinsic geometric structure of the data.  
To obtain a closed-form expression for the mapping matrix \( W_{c}\) during the optimisation process, we constrain it to the Stiefel manifold \( St(n,m)\)~\cite{huang2017riemannian}. We define \( c \) geometric frameworks as mapping matrices \(\left\{ W_{1}, \ldots, W_{c}\right\}\) capturing different subspace representations, where \( c \) denotes the number of geometric frameworks. The learnability of the mapping matrix enables the model to adaptively adjust the geometric framework, capture the geometric information of multiple subspaces, and improve modelling capacity for geometric feature extraction and subspace fusion.

\subsubsection{Grassmannian Subspace Interaction (GSI)}
In the GMSR layer, we construct subspace representations under various geometric frameworks $c$. 
The input is a set of $m'$ Grassmannian subspace points denoted as \(\{P_i\}_{i=1}^{m'}\), where each \(P_i = X_{GMSR}^{c,i} \in \mathcal{G}(m, p)\).
To enable effective fusion of the set \(\{P_i\}_{i=1}^{m'}\), we introduce the Fréchet mean to preserve the intrinsic geometric structure~\cite{pennec2006riemannian}. 
Specifically, Fréchet mean $\Gamma$ is computed as the point that minimises the sum of squared distances on the manifold~\cite{yang2010riemannian}:
\begin{equation}
\Gamma\!=\!\mathcal{F}_w(\{P_{i}\}_{\!i\leq m'})\!:=\!\textit{arg}\min_{\!\Gamma\in\mathcal{G}(m,p)}\sum_{i=1}^{m'}\!w_i\delta_R^2(\Gamma,P_{i}),
\end{equation}
where, $w_i$ denotes the weight of each point, satisfying $w_i \geq 0$ and $\sum_{i=1}^{m'} w_i = 1$. 
The function $\delta_R(\cdot, \cdot)$ represents the geodesic distance on the Grassmannian.

When \( m' = 2 \), \textit{i.e.}, the weights are given by \( \{w, 1 - w\} \), a closed-form solution exists, which exactly corresponds to the geodesic between the two points \( P_1 \) and \( P_2 \), parameterised by \( w \in [0, 1] \):
\begin{equation}
\Gamma = V \left[ P_1 \cos(w \Theta) + U \sin(w \Theta) \right] V^\top,
\end{equation}
where \( P_1^\top P_2 = V \cos(\Theta) Q^\top \) is part of its singular value decomposition (SVD), \( U, \Theta, V \) come from the SVD of \( (I - P_1P_2^\top) P_2 V \), and the parameter \( w \in [0, 1] \) controls the position of the fused point~\cite{beik2021learning}.

When \( m' > 2 \), the optimisation problem does not have a closed-form solution. 
Therefore, the Karcher flow algorithm is employed to perform optimisation in the tangent space iteratively, ultimately converging to the Fréchet mean through exponential mapping. 
The update formula of the Karcher flow~\cite{karcher1977riemannian,yang2010riemannian} algorithm is as follows (Appendix C provides detailed steps):
\begin{equation}
\Gamma_{new} = \exp_{\Gamma}\left(\Gamma, \alpha \cdot \frac{1}{m'} \sum_{i=1}^{m'} \log_{\Gamma}(\Gamma, P_i) \right),
\end{equation}
where \( \alpha \) is the step size, \( \Gamma \) is the current mean, and \( P_i \) represents each subspace within the same geometric framework. 
\( \log_G(\cdot, \cdot) \) and \( \exp_G(\cdot, \cdot) \) are provided in the Appendix A.
\subsection{Deep Subspace Interaction and Optimisation Principles}
\label{Optimization}
\textbf{Deep Subspace Interaction Principles:}
We propose a stackable GMSF-Net by applying multiple blocks to meet the flexible and diverse requirements of practical applications. 
The core principle is to constrain the feature manifold distribution through Riemannian batch normalisation, thereby enhancing geometric discriminability while stabilising the feature space topology. 
The feature construction based on Riemannian BN provides navigation signals with geometric interpretability, enabling the stacking structure to systematically enhance and refine the features.

For the input of the AMSE block, \textit{i.e.} statistical features, we introduce a Riemannian normalisation strategy~\cite{brooks2019riemannian} based on manifold geometry, mapping statistical features to the SPD manifold~\cite{nguyen2019neural} and normalising them
to extract more discriminative key Riemannian features (see Appendix E for the detailed implementation of SPDBN):
\begin{equation}
\hat{X} = \exp_{\hat{B}} \left( \hat{B}^{1/2} \left( (\hat{B}\!B)^{-1/2} X (\hat{B}\!B)^{-1/2} \right) \hat{B}^{1/2} \right),
\end{equation}
where \( \hat{X} \) is the refined feature, \( X \) is the original input, \( B \) is the Riemannian barycenter, and \( \hat{B} \) is the learnable common feature, which is iteratively updated using the Riemannian geodesic formula~\cite{yair2019parallel}.

\textbf{Optimisation Strategies:} 
In order to effectively train the proposed GMSF-Net, we restructure the loss function by combining classification cross-entropy loss~\cite{hinton1995wake} and MI-inspired regularizer loss, effectively capturing diverse geometric subspaces. 
The total loss function for the proposed method is as follows:
\begin{equation}
\mathcal{L}_{total} = \mathcal{L}_{CE} + \lambda \cdot \mathcal{L}_{C} ,
\end{equation}
where, \(\mathcal{L}_{total}\) is the total loss, \(\mathcal{L}_{CE}\) is the cross-entropy loss, \(\mathcal{L}_{C}\) is the MI-inspired regularizer loss, and \(\lambda\) is a hyperparameter that balances classification accuracy with structural preservation. 
The details of these two loss functions (Appendix D.1) and their backpropagation derivations (Appendix D.2) are provided in Appendix D.

\section{Evaluation}
\label{Experiment}
In this section, we evaluate the performance of the proposed GMSF-Net on two challenging classification tasks:  HDM05~\cite{muller2007mocap} and FPHA~\cite{garcia2018first} datasets for video-based 3D action recognition, and the MAMEM-SSVEP-II~\cite{pan2022matt} dataset for EEG signal classification. 
Additionally, we perform node classification (NC) and link prediction (LP) tasks on graph datasets, where the underlying structure of the data can naturally be represented as Grassmannian, possessing multi-subspace geometric properties. 
More experimental details are presented in Appendix G.

\begin{table}[t]
  \centering
  \small
  \setlength{\tabcolsep}{4pt}
  \begin{tabular}{llcc}
    \toprule
    Method     &Acc.(\%)     &Size($MB$)     &FLOPs($M$)  \\
    \midrule
    SPDNet & 87.65\%±1.02  & 13.60     & 1595.50 \\
    GrNet  & 78.79\%±1.82 & 6.73   & 38.60    \\
    MATT     & 87.70\%±0.68       & 1.83   & 142.07\\
    SPDNetBN     & 89.33\%±0.49       & 13.63  & 1902.97 \\
    GDLNet     & 87.60\%±0.69       & 1.83   & \textbf{33.69}\\
    \midrule
    GMSF-Net-1Block     & 90.43\%±0.74       & \textbf{1.20}  & 48.42 \\
    GMSF-Net-2Blocks     & 90.70\%±0.70       & 1.26  & 63.44 \\
    GMSF-Net-3Blocks     & \textbf{91.22\%±0.53}      & 1.30  & 81.07 \\
    \bottomrule
  \end{tabular}
  \caption{Comparison Between GMSF-Net and Other Riemannian Solutions on the FPHA Dataset.}
  \label{FPHA} 
\end{table}
\begin{table}[t]
  \centering
  \small
  \setlength{\tabcolsep}{4pt}
  \begin{tabular}{llcc}
    \toprule
    Method     &Acc.(\%)     &Size($MB$)     &FLOPs($M$)  \\
    \midrule
    SPDNet & 60.45\%±1.12  & 11.71     & 2050.49\\
    GrNet  & 59.23\%±1.78 & 6.88   & \textbf{81.96}   \\
    MATT     & 62.25\%±1.68       & 3.99 & 442.45 \\
    GDLNet     & 60.08\%±1.78       & 3.95  & 89.11\\
    \midrule
    GMSF-Net-1Block     & 63.64\%±1.24       & \textbf{3.33} & 96.08 \\
    GMSF-Net-2Blocks     & 63.98\%±1.07       & 3.34  & 100.07\\
    GMSF-Net-3Blocks     & \textbf{64.19\%±0.88}       & 3.39 & 114.12 \\
    \bottomrule
  \end{tabular}
  \caption{Comparison Between GMSF-Net and Other Riemannian Solutions on the HDM05 Dataset.}
  \label{HDM05}
\end{table}

\begin{table}[t]
  \centering
  \small
  \setlength{\tabcolsep}{3pt}
  \begin{tabular}{lccc}
    \toprule
    Method & Acc.(\%) & Size($MB$) & FLOPs ($M$) \\
    \midrule
    EEGNet           & 53.72\%±7.23 & \textbf{0.075} & 60.13 \\
    ShallowConvNet   & 56.93\%±6.97 & 0.18  & 127.00 \\
    EEG-TCNet       & 55.45\%±7.66 & 0.016 & 60.80 \\
    SCCNet              & 62.11\%±7.70 & 0.55  & 108.56 \\
    MBEEGSE        & 56.45\%±7.27 & 1.64  & 52.40 \\
    FBCNet             & 53.09\%±5.67 & 0.24  & \textbf{51.80} \\
    \midrule  
    SPDNet                                     & 62.30\%±3.12 & 2.81  & 2274.56 \\
    GrNet                                      & 61.23\%±3.56 & 1.95  & 2020.78 \\
    MATT                   & 65.19\%±3.14 & 1.97  & 2068.44 \\
    SPDNetBN                                   & 62.76\%±3.01 & 2.81  & 2312.78 \\
    GDLNet                                     & 65.52\%±2.86 & 1.95  & 2028.34 \\
    \midrule
    GMSF-Net-1Block                            & 66.74\%±1.79 & 1.94  & 2020.14 \\
    GMSF-Net-2Blocks                           & 66.32\%±1.84 & 1.94  & 2021.78 \\
    GMSF-Net-3Blocks                           & \textbf{66.87\%±1.46} & 1.94 & 2032.55 \\
    \bottomrule
  \end{tabular}
  \caption{Performance comparison between GMSF-Net and the baseline on SSVEP.}
  \label{SSVEP}
\end{table}

\begin{table*}[ht]
  \centering
  \small
  \setlength{\tabcolsep}{5pt}
  \begin{tabular}{clcccccccc}
    \toprule
    \multicolumn{2}{c}{\textbf{Dataset}} & \multicolumn{2}{c}{\textbf{Disease}} & \multicolumn{2}{c}{\textbf{Airport}} & \multicolumn{2}{c}{\textbf{PubMed}} & \multicolumn{2}{c}{\textbf{CoRA}} \\
    \cmidrule(lr){1-2} \cmidrule(lr){3-4} \cmidrule(lr){5-6} \cmidrule(lr){7-8} \cmidrule(lr){9-10} 
    & \textbf{Task} & \textbf{LP} & \textbf{NC} & \textbf{LP} & \textbf{NC} & \textbf{LP} & \textbf{NC} & \textbf{LP} & \textbf{NC} \\
    \midrule
    
    & Euc          & 59.8\%±2.0 & 32.5\%±1.1 & 92.0\%±0.0 & 60.9\%±3.4 & 83.3\%±0.1 & 48.2\%±0.7 & 82.5\%±0.3 & 23.8\%±0.7 \\
    & Hyp    & 63.5\%±0.6 & 45.5\%±3.3 & 94.5\%±0.0 & 70.2\%±0.1 & 87.5\%±0.1 & 68.5\%±0.3 & 87.6\%±0.2 & 22.0\%±1.5 \\
    & Euc-Mixed    & 49.6\%±1.1 & 35.2\%±3.4 & 91.5\%±0.1 & 68.3\%±2.3 & 86.0\%±1.3 & 63.0\%±0.3 & 84.4\%±0.2 & 46.1\%±0.4 \\
    & Hyp-Mixed    & 55.1\%±1.3 & 56.9\%±1.5 & 93.3\%±0.0 & 69.6\%±0.1 & 83.8\%±0.3 & \textbf{73.9\%}±0.2 & 85.6\%±0.5 & 45.9\%±0.3 \\
    & MLP          & 72.6\%±0.6 & 28.8\%±2.5 & 89.8\%±0.5 & 68.6\%±0.6 & 84.1\%±0.9 & 72.4\%±0.2 & 83.1\%±0.5 & 51.5\%±1.0 \\
    & HNN   & 75.1\%±0.3 & 41.0\%±1.8 & 90.8\%±0.2 & 80.5\%±0.5 & 94.9\%±0.1 & 69.8\%±0.4 & 89.0\%±0.1 & 54.6\%±0.4 \\
    \midrule
    & GMSF-Net-1Block & 95.0\%±1.0 & 79.7\%±2.9 & \textbf{93.4\%}±0.2 & \textbf{82.5\%}±0.8 & 94.6\%±0.1 & 72.7\%±0.5 & 88.6\%±0.6 & \textbf{57.8\%}±1.2 \\
    & GMSF-Net-2Blocks & 94.6\%±0.7 & 81.1\%±1.9 & 92.7\%±0.6 & 80.7\%±1.0 & \textbf{95.0\%}±0.1 & 72.8\%±0.4 & \textbf{89.3\%}±0.6 & 55.8\%±1.5 \\
    & GMSF-Net-3Blocks & \textbf{95.5\%}±0.4 & \textbf{82.7\%}±0.8 & 91.2\%±0.2 & 78.3\%±0.6 & 93.6\%±0.3 & 72.9\%±0.3 & 88.9\%±0.3 & 56.2\%±1.1 \\
    \bottomrule
  \end{tabular}
  \caption{Performance comparison between GMSF-Net and the baselines on Graph Tasks.}
  \label{tab:Graph}
\end{table*}

\begin{table*}[ht]
  \centering
  \small
  \setlength{\tabcolsep}{5pt}
  \begin{tabular}{lcccccc}
    \toprule
    \multicolumn{1}{c}{} & \multicolumn{3}{c}{\textbf{With Interaction}} & \multicolumn{3}{c}{\textbf{Without Interaction}} \\
    \cmidrule(lr){2-4} \cmidrule(lr){5-7}  
    \textbf{Configuration} & \textbf{HDM05} & \textbf{FPHA} & \textbf{SSVEP} & \textbf{HDM05} & \textbf{FPHA} & \textbf{SSVEP} \\
    \midrule 
    Adaptive-Subspace   & \textbf{63.64\%}±1.24 & \textbf{90.43\%}±0.74 & \textbf{66.74\%}±1.79  & \textbf{56.49\%}±1.37 & \textbf{80.68\%}±2.15 & \textbf{59.83\%}±2.72 \\
    Random-Subspace     & 50.29\%±2.10 & 72.47\%±3.19 & 56.01\%±2.46 & 52.50\%±1.82 & 73.10\%±2.19 & 57.43\%±1.10 \\
    Fixed-Subspace      & 53.04\%±0.88 & 83.06\%±0.61 & 66.05\%±1.95 & 51.12\%±3.55 & 76.67\%±0.16 & 58.85\%±1.19  \\
    \bottomrule
  \end{tabular}
   \caption{Performance Comparison of Subspace Construction and Interaction Mechanisms.}
   \label{tab:adaptive}
\end{table*}

\subsection{Experimental Evaluation on the Classification Datasets}
We evaluate the proposed GMSF-Net on video-based 3D action recognition and EEG signal classification tasks. 
For action recognition, we use the FPHA and HDM05 datasets, following the experimental settings of SymNet~\cite{wang2021symnet} and GrNet~\cite{huang2018building}. 
For EEG classification, we follow the preprocessing pipeline of GDLNet~\cite{wang2024grassmannian} and conduct experiments on the MAMEM-SSVEP-II dataset.
In the model evaluation, we compare performance involving different numbers of multi-subspace interaction blocks. 
As performance saturates with more blocks, we report results with up to three blocks in this section. 
For each task, we highlight the best performance results in bold.

\textbf{FPHA:} Following~\cite{wang2021symnet}, we evaluate GMSF-Net using three network configurations: 1Block, 2Blocks, and 3Blocks. 
The results from 10 random experiments (mean ± std) are shown in  Table \ref{FPHA}. 
Notably, GMSF-Net consistently outperforms most Riemannian deep models, with a \textbf{12.43\%} improvement over GrNet. 
As the number of blocks increases, performance improvements plateau, yet the effectiveness of the adaptive subspace interaction method remains evident. 
Additionally, our model demonstrates significant advantages in terms of model size and computational cost.

\textbf{HDM05:} Following~\cite{huang2018building}, we also adopt three architectures to evaluate GMSF-Net: 1Block, 2Blocks, and 3Blocks configurations. 
The 10-fold results (mean ± std) are presented in Table \ref{HDM05}. 
It should be noted that GMSF-Net achieves a performance improvement of up to \textbf{4.96\%} over GrNet. 
This highlights the effectiveness of our approach and significantly reduces memory and computational costs compared to SPD manifold-based neural networks. 
Results on HDM05 and FPHA demonstrate that GMSF-Net consistently outperforms other methods in 3D action recognition.

\textbf{MAMEM-SSVEP-II:} In Table \ref{SSVEP}, we follow the protocol of GDLNet and report the results of ten random cross-validation runs (mean ± std). 
It can be observed that GMSF-Net outperforms existing Riemannian models as well as EEG deep learning models, including EEGNet~\cite{lawhern2018eegnet}, ShallowConvNet~\cite{schirrmeister2017deep}, EEG-TCNet~\cite{ingolfsson2020eeg}, SCCNet~\cite{wei2019spatial}, MBEEGSE~\cite{altuwaijri2022multi}, and FBCNet~\cite{mane2021fbcnet}, in terms of overall performance. 
The accuracy is improved by \textbf{5.64\%} over GrNet, and by \textbf{1.68\%} and \textbf{1.35\%} over the Riemannian self-attention models MATT~\cite{pan2022matt} and GDLNet, with model stability also enhanced. EEG visualisation results are shown in Appendix G.4.
These results further highlight the versatility and effectiveness of our framework.
\subsection{Experimental Evaluation on Graph Datasets}
To evaluate the broad applicability of GMSF-Net, we perform node classification (NC) and link prediction (LP) on graph datasets with prominent Grassmannian subspace structures.
The implementation details is provided in Appendix G.2.2.
The experimental results show that increased deviation from the ideal Grassmannian structure leads to performance degradation due to modelling limits on non-natural manifold geometries. 
In contrast, GMSF-Net exhibits greater performance gains on datasets more aligned with the Grassmannian assumption, validating the soundness and adaptability of its design. 
Specifically, GMSF-Net achieves notable gains on Cora~\cite{sen2008collective} and Disease~\cite{chami2019hyperbolic} datasets, which exhibit clear subspace structures, while improvements are more modest on structurally complex PubMed~\cite{sen2008collective} and Airport~\cite{chami2019hyperbolic} datasets that deviate from ideal subspace assumptions. 
GMSF-Net achieves notable gains on all graph datasets except PubMed, as shown in Table \ref{tab:Graph}.
\subsection{Ablation Experiments and Analysis}
As shown in Table \ref{tab:adaptive}, we conducted six ablation studies across three datasets to evaluate the effectiveness of the adaptive subspace and interaction mechanisms. 
The results show that the framework combining adaptive subspace construction with the interaction mechanism consistently achieves the best performance (HDM05: 63.64\%, FPHA: 90.43\%, SSVEP: 66.74\%), significantly outperforming other configurations.
This validates the necessity of adaptive mechanisms for optimising the feature space and highlights the synergistic enhancement provided by the interaction mechanism within high-confidence subspaces. 
As shown in Table 5, the random subspace achieves the lowest performance (\textit{e.g.}, HDM05: 52.50\%~$\rightarrow$~50.29\%), indicating that low-quality interactions may introduce noise. 
While the fixed subspace reaches 83.06\% on FPHA, its static structure limits its ability to adapt to diverse geometric characteristics. 
The adaptive subspace is essential for effectively collaborating with the interaction mechanism to capture feature correlations and enhance generalisation performance. 
Additional experiments on model hyperparameters are provided in Appendix G.3.

\section{Conclusion}
\label{Conclusion}
This study innovatively proposes a topology-driven multi-subspace fusion framework on the Grassmannian, unifying adaptive geometric representation learning with differentiable subspace interaction. 
By dynamically optimising subspaces to generate task-oriented geometric representation primitives and incorporating a geometry-aware interaction network for subspace topology correlation modelling, we successfully achieve simultaneous improvements in model accuracy and noise robustness in cross-dataset experiments. 
This work provides a new differentiable manifold learning paradigm for fields such as computer vision and natural language processing.

\section{Acknowledgement}
This work was supported in part by the National Natural Science Foundation of China (62576152, 62336004, 62020106012), the Basic Research Program of Jiangsu (BK20250104), and the Fundamental Research Funds for the Central Universities (JUSRP202504007)

\bibliography{aaai2026}

\section*{Appendix}


\section{A Preliminaries}
\label{Preliminaries}
\subsection{A.1 Brief Introduction to Riemannian Geometry}
Intuitively, manifolds can be viewed as local extensions of Euclidean spaces. In these geometric structures, differentials serve as a natural generalization of derivatives in classical calculus. When a manifold is endowed with a Riemannian metric, where each tangent space at a point is equipped with an inner product, it becomes a Riemannian manifold. For more details on smooth manifolds, please refer to~\cite{boumal2023introduction}.
\begin{definition}[Riemannian Manifold]
A Riemannian metric on a manifold \(M\) is a smooth, symmetric, covariant 2-tensor field on \(M\) that is positive definite at every point. A Riemannian manifold is a pair \(\{M, g\}\), where \(M\) is a smooth manifold and \(g\) is a Riemannian metric.
\end{definition}
For simplicity, we abbreviate \(\{M, g\}\) as \(M\). The Riemannian metric \(g\) induces various Riemannian operators, including geodesics, exponential maps, logarithmic maps, and parallel transport. These operators correspond to straight lines, vector addition, vector subtraction, and parallel displacement in Euclidean spaces, respectively. Detailed discussions on Riemannian geometry can be found in~\cite{boumal2023introduction}.

Additionally, we need to understand the concept of pullback metrics, an important technique in Riemannian geometry, which is a natural extension of the bijection concept from set theory.

\begin{definition}[Pullback Metrics]
Let \(M\) and \(N\) be smooth manifolds, \(g:\) a Riemannian metric on \(N\), and \(f: M \rightarrow N\) a smooth map. The pullback of \(g\) by \(f\) is defined pointwise as follows:
\begin{equation}
(f^* g)_p(V_1, V_2) = g_{f(p)}(f^*_p(V_1), f^*_p(V_2)),
\end{equation}
where \(p \in M\), \(f^*_p(\cdot)\) is the differential map of \(f\) at \(p\), and \(V_i \in T_p M\). If \(f^*g\) is positive definite, it is a Riemannian metric on \(M\), called the pullback metric defined by \(f\).
\end{definition}
\subsection{A.2 Basic geometries of Grassmannian manifolds}
\label{geometries}
The Grassmannian manifold \( \mathcal{G}(n, p) \) is the set of all \( p \)-dimensional linear subspaces of \( \mathbb{R}^n \). It is a Riemannian manifold with \( p(n - p) \) dimensions. Each point on \( \mathcal{G}(n, p) \) represents a \( p \)-dimensional subspace, which can be represented by an orthonormal basis matrix \( U \):
\begin{equation}
\mathcal{G}(n, p) := \left\{ U \in \mathbb{R}^{n \times p} \mid U^T U = I_p \right\},
\end{equation}
where \( U^T \) denotes the transpose of \( U \) and \( I_p \) is the identity matrix of size \( p \times p \). The columns of the matrix \( p \) are orthogonal~\cite{bendokat2024grassmann}.

\textbf{Tangent Space of the Grassmannian Manifold:}
To study differential geometry on the Grassmannian manifold, we need to define the tangent space. The tangent space characterizes the "possible directions" near a given point, that is, the directions along which one can "move" starting from that point.
\begin{definition}[\textbf{Tangent Space in the Orthonormal Basis View}]  
At the point $\boldsymbol{U} \in \mathcal{G}(n, p)$, the \textbf{tangent space} $T_{\boldsymbol{U}}\mathcal{G}(n, p)$ is defined as:
\begin{equation}
T_{\boldsymbol{U}}\mathcal{G}(n, p) = \left \{ \boldsymbol{\Delta} \in \mathbb{R}^{n \times p} \ \middle| \ \boldsymbol{U}^\top \boldsymbol{\Delta} = \boldsymbol{0} \right \}.
\end{equation}
\end{definition}

\textbf{Exponential Map and Log Map of the Grassmannian Manifold:} 
On the Grassmannian manifold, the exponential map and the logarithmic map are important tools for connecting the tangent space and the manifold. The exponential map maps vectors from the tangent space to points on the manifold, while the logarithmic map is its inverse operation.
\begin{definition}[\textbf{Exponential Map in the Orthogonal Basis Viewpoint}]
The exponential map $\operatorname{Exp}_{\boldsymbol{U}}$ maps a tangent vector $\boldsymbol{\Delta} \in T_{\boldsymbol{U}}\mathcal{G}(n, p)$ to a point on the manifold, defined as:
\begin{equation}
\operatorname{Exp}_{\boldsymbol{U}}(\boldsymbol{\Delta}) = \boldsymbol{U} \boldsymbol{R} \cos(\boldsymbol{\Sigma}) + \boldsymbol{Q} \sin(\boldsymbol{\Sigma}),
\end{equation}
where, $\boldsymbol{\Delta} = \boldsymbol{Q} \boldsymbol{\Sigma} \boldsymbol{R}^\top$ is the singular value decomposition of the tangent vector.
\end{definition}

\begin{definition}[\textbf{Logarithmic Map in the Orthogonal Basis Viewpoint}]
For $\boldsymbol{U}, \boldsymbol{Y} \in \mathcal{G}(n, p)$, the Riemannian logarithmic map $\operatorname{Log}_{\boldsymbol{U}}(\boldsymbol{Y})$ is defined as:
\begin{equation}
\operatorname{Log}_{\boldsymbol{U}}(\boldsymbol{Y}) = \boldsymbol{O} \arccos(\boldsymbol{\Sigma}) \boldsymbol{R}^\top,
\end{equation}
where
\begin{equation}
\left( \boldsymbol{I}_n - \boldsymbol{U}\boldsymbol{U}^\top \right) \boldsymbol{Y} \left( \boldsymbol{U}^\top \boldsymbol{Y} \right)^{-1} = \boldsymbol{O} \boldsymbol{\Sigma} \boldsymbol{R}^\top
\end{equation}
is the singular value decomposition of the matrix.
\end{definition}

\subsection{A.3 Basic Concepts in Topology}
\label{Topology}
\begin{definition}[\textbf{Topology Space Definition}]
A set \( X \) and a collection of subsets \( \mathcal{T} \) are called a topological space if they satisfy the following conditions: 

\begin{itemize}
    \item The empty set and the set \( X \) are in \( \mathcal{T} \), \textit{i.e.}, \( \varnothing \in \mathcal{T} \) and \( X \in \mathcal{T} \). 
    \item The union of any collection of sets from \( \mathcal{T} \) is in \( \mathcal{T} \), \textit{i.e.}, \( \bigcup_{i \in I} U_i \in \mathcal{T} \), for all \( U_i \in \mathcal{T} \) and \( i \in I \). 
    \item The intersection of any finite collection of sets from \( \mathcal{T} \) is in \( \mathcal{T} \), \textit{i.e.}, \( \bigcap_{i=1}^n U_i \in \mathcal{T} \), for all \( U_i \in \mathcal{T} \) and \( i = 1, 2, \dots, n \).
\end{itemize}
\end{definition}
\begin{definition}[\textbf{Metric Topology Definition}]
\label{def:metrictopology}
Let \( (X, d) \) be a metric space, where \( X \) is a set and \( d: X \times X \to \mathbb{R} \) is a metric function satisfying the following four conditions:
\begin{enumerate}
    \item \textbf{Non-negativity}: For any \( x, y \in X \), \( d(x, y) \geq 0 \).
    \item \textbf{Zero distance equivalence}: \( d(x, y) = 0 \) if and only if \( x = y \).
    \item \textbf{Symmetry}: For any \( x, y \in X \), \( d(x, y) = d(y, x) \).
    \item \textbf{Triangle inequality}: For any \( x, y, z \in X \), \( d(x, z) \leq d(x, y) + d(y, z) \).
\end{enumerate}
In the metric space \( (X, d) \), the metric topology is the topology induced by the metric \( d \), with open sets defined as follows: a set \( U \subseteq X \) is open if and only if for every point \( x \in U \), there exists a positive number \( \varepsilon > 0 \) such that the open ball \(B_d(x, \varepsilon) = \{ y \in X \mid d(x, y) < \varepsilon \}\) is entirely contained in \( U \). In other words, the open sets in the metric topology are those that contain all open balls.
\end{definition}

\section{B The Hierarchical Structure of GrNet}
\label{GrNet}
\paragraph{FRMap:}
In GrNet, we introduce the FRMap layer to learn compact and discriminative Grassmannian representations. This layer transforms the input orthonormal subspace matrices into new matrices via a linear mapping function $f_{fr}$:
\begin{equation}
X_k = f_{fr}(X_{k-1}, W_k) = W_k X_{k-1},
\end{equation}
where $X_{k-1} \in \mathcal{G}(d_{k-1}, q)$ is the input of the $k$-th layer, $W_k \in \mathbb{R}^{d_k \times d_{k-1}}$ ($d_k < d_{k-1}$) is the transformation matrix required to be row full-rank, and $X_k \in \mathbb{R}^{d_k \times q}$ is the output matrix. Since $W_k X_{k-1}$ is generally not an orthonormal basis matrix, we apply QR decomposition in the subsequent ReOrth layer to ensure orthogonality.

The weight space $\mathbb{R}^{d_k \times d_{k-1}}$ of the FRMap layer belongs to a non-compact Stiefel manifold, where the geodesic distance lacks a closed-form solution, making direct optimisation infeasible. To address this, we impose an orthogonality constraint on $W_k$, restricting it to the compact Stiefel manifold $\mathcal{S}{t}(d_k, d_{k-1})$, ensuring stability and convergence during optimisation.
\paragraph{OrthMap:}
The main purpose of the OrthMap layer is to extract the orthogonal representation of the subspace on the Grassmannian manifold from the original data, thereby transforming the data in a way that is suitable for subsequent processing while preserving the properties of the Grassmannian manifold. This transformation is achieved through Eigenvalue Decomposition (EIG). Specifically, for the input projection matrix \( X_k \), the Eigenvalue Decomposition is performed as follows:
\begin{equation}
X_k = U_k \Sigma_k U_k^T,
\end{equation}
where \( U_k \) is an orthogonal matrix containing the eigenvectors, and \( \Sigma_k \) is a diagonal matrix containing the corresponding eigenvalues. The OrthMap layer then selects the eigenvectors corresponding to the top \( q \) largest eigenvalues from \( U_k \), forming a new orthogonal matrix \( U_{k-1,1:q} \). This matrix becomes the output of the layer, expressed as:
\begin{equation}
X_{k+1} = f_{\text{om}}(X_k) = U_{k-1,1:q},
\end{equation}
Here, \( U_{k-1,1:q} \) is the orthogonal matrix formed by the eigenvectors of \( U_k \) corresponding to the largest \( q \) eigenvalues, ensuring that the output remains orthogonal and optimising the data representation on the Grassmannian manifold.

\paragraph{ReOrth:}

The design of the ReOrth layer is inspired by QR decomposition, which is used to convert a regular matrix into an orthogonal matrix. Specifically, the ReOrth layer achieves this by performing a QR decomposition on the input matrix \( X_{k-1} \). The QR decomposition breaks \( X_{k-1} \) into a product of two matrices: an orthogonal matrix \( Q_{k-1} \) and an invertible upper triangular matrix \( R_{k-1} \).

Through QR decomposition, we can transform \( X_k \) into an orthogonal basis matrix. Specifically, at the \( k \)-th layer, by normalizing \( X_{k-1} \), we obtain:
\begin{equation}
X_k = X_{k-1} R_{k-1}^{-1} = Q_{k-1}.
\end{equation}
In the context of convolutional neural networks (ConvNets), several nonlinear activation functions, such as the Rectified Linear Unit (ReLU), have been proposed to enhance discriminative performance. Therefore, incorporating such nonlinearity into the proposed GrNet is also essential. In this case, the ReOrth layer also serves to execute nonlinear activation through QR decomposition.

\paragraph{ProjMap:}
The ProjMap layer is a key component in the GrNet network architecture, designed to enable efficient mapping between the Grassmannian and Euclidean spaces. This layer employs a projection metric, a common choice in Grassmannian metrics, which endows a specific Riemannian manifold with an inner product structure, simplifying the manifold into a flat space. In this flat space, classical Euclidean computational methods can be directly applied to the projection domain of orthogonal matrices.

Specifically, the ProjMap layer applies a projection mapping to the orthogonal matrix \( X_{k-1} \) at the \( k \)-th layer, with the following computation:
\begin{equation}
X_k = f_{\text{pm}}(X_{k-1}) = X_{k-1} X_{k-1}^T.
\end{equation}
This operation maps the orthogonal matrix to its projection matrix, allowing subsequent computations to be performed in the Euclidean space. This approach not only preserves the geometric properties of the Grassmannian manifold but also leverages mature computational techniques from Euclidean space, significantly improving processing efficiency.

\section{C Details of Multi-Subspace Interaction}
\label{Mean}
We provide a brief introduction to the optimisation process of multi-subspace fusion when the number of subspaces \( m' > 2 \). In this case, the Karcher flow algorithm is typically employed to iteratively compute the Fréchet mean. An illustration of one iteration of the algorithm is shown in~\cref{Fréchet}.
\begin{figure}[ht]
\centering
\includegraphics[width=0.5\linewidth,trim={18 30 20 120},clip=true]{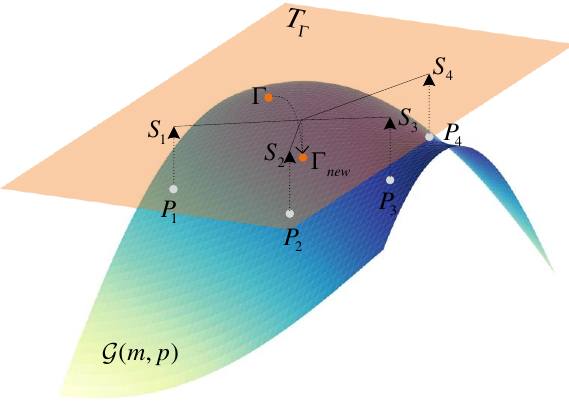}
\caption{\textit{Illustration of one iteration of the fusion between multiple subspaces.} The data points \( P_i \), each representing a different subspace, are logarithmically mapped to the tangent space at \( \Gamma \). These subspaces \( S_i \) are then arithmetically averaged. The result is exponentially mapped back to the manifold, yielding the updated mean \( \Gamma_{\text{new}} \).}
\label{Fréchet}
\end{figure}

\section{D Loss and Optimisation on the Grassmannian Manifold}
\label{Loss and Optimization}
\subsection{D.1 Loss}
\label{Loss}
Cross-entropy loss \(\mathcal{L}_{CE}\) is a common loss function in classification problems, measuring the difference between the predicted probability distribution and the true labels. The mutual information loss \(\mathcal{L}_{C}\) enhances the model's sensitivity to Grassmannian subspaces by maximizing the information complementarity between features~\cite{cheng2020club}. The form of the mutual information loss \(\mathcal{L}_{C}\) is as follows:
\begin{equation}
\mathcal{L}_{C} = - \sum_{i=1}^{n} \sum_{j=1}^{n} \log \left( \frac{\text{cov}(X_i, X_j)}{\sqrt{\text{var}(X_i) \cdot \text{var}(X_j)}} \right).
\end{equation}
In the above equation, \(\text{cov}(X_i, X_j)\) represents the covariance between the Grassmannian subspaces \(X_i\) and \(X_j\), and \(\text{var}(X_i)\) and \(\text{var}(X_j)\) represent the variances of the Grassmannian subspaces \(X_i\) and \(X_j\), respectively. The mutual information loss maximizes the covariance between Grassmannian subspaces while minimizing redundant information, thus encouraging complementarity between the selected subspaces and improving the model's discriminative capability.
\subsection{D.2 Backward Propagation}
\label{Backward Propagation}
In the GMSF-Net, backpropagation must take into account the geometric structure of the manifold to maintain the accuracy of parameter updates. This is especially important for adaptive multi-subspace construction and multi-subspace interaction blocks in the deep Grassmannian network, where the gradient calculation needs to consider both the effectiveness and the constraints of the gradient~\cite{huang2018building}.

\subsubsection{Adaptive Multi-Subspace Construction}

In AdaMSC section, we provide a detailed description of the adaptive multi-subspace construction and introduce $W_c$, where each row corresponds to a newly formed subspace weight. The values within each row represent the adaptive weights assigned to the internal subspaces. We begin by computing the derivative of $L^{(k)}$ with respect to $X$~\cite{huang2017riemannian}.
\begin{equation}
\frac{\partial L^{(k)}}{\partial X} = ([\frac{\partial L^{(k')}}{\partial K}\quad0])^T \cdot\sigma_{c}\cdot W^{'}_c ,
\end{equation}
When a virtual layer \( k' \) is introduced as the subsequent layer of \( k \) in the forward pass, the gradient at layer \( k \) is given by \( \partial L^{(k)} = \ell \circ f^{(l)} \circ \cdots \circ f^{(k)} \), where \( \ell \) denotes the loss function and \( l \) is the final layer of the network. The matrices \( \sigma_c \) is binary matrices extended from \( W_c \). For \( W_c \), the top \( p \) positions with the largest values in each row are marked as \( 1 \) in the corresponding columns in the specific channels of \( \sigma_c \), while the remaining positions are marked as \( 0 \), \( W^{\prime}_c \) represents the descendingly sorted form of \( W_c \), in which the elements of each row are arranged from largest to smallest. Below is the calculation of the derivatives of \( L^{(k)} \) with respect to \( W_c \).
\begin{equation}
\frac{\partial L^{(k)}}{\partial W_c} = \sum_{i=1}^{B} \sum_{j=1}^{n}[\frac{\partial L^{(k')}}{\partial K}\quad0][i,j]\cdot \frac{\partial K[m,n]}{\partial W_c}\cdot X,
\end{equation}
Where \( B \) denotes the number of newly formed subspaces, \( p \) is the number of selected column subspaces, and \( n \) is the number of original subspaces. The calculation also involves the derivative of the intermediate matrix \( K \), and the relationship between \( K \) and the weight matrices \( W_c \) is shown in the following equations.
\begin{equation}
\frac{\partial K[m,n]}{\partial W_c} = 
\begin{cases} 
K[m,n](1 - K[m,n]) & \text{if } m = n \\
-K[m,n] \cdot K[n] & \text{if } m \neq n
\end{cases}
\end{equation}

\subsubsection{Multi-Subspace Interaction Block}

In Multi-Subspace section, we describe the construction of the multi-subspace block, which mainly relies on the operation of the GMSR and the AdaMSC, and provide the derivative of $L^{(k)}$ with respect to $X$.
\begin{equation}
\frac{\partial L^{(k)}}{\partial X} = ([(\frac{\partial L^{(k')}}{\partial K})\mathit{W}^{T}_k\quad0])^T \cdot\sigma_{c}\cdot W^{'}_c,
\end{equation} 
Where, \( \frac{\partial L^{(k')}}{\partial K} W_k^{T} \) represents the gradient contribution of the GMSR to \( X \), where \( W_k \) is the projection matrix. By further multiplying with \( \sigma_c \), the influence of the adaptive layer is incorporated, enabling a joint optimisation effect of the GMSR and the AdaMSC on \( X \).

\section{E Riemannian Batch Normalization}
\label{batch}
Riemannian batch normalization normalizes the data distribution on the SPD manifold, avoiding problems like gradient explosion and vanishing gradients during training~\cite{ioffe2015batch}. The specific form of SPD manifold batch normalization is given by:
\begin{equation}
 \Gamma_{\mathcal{X} \rightarrow \mathcal{B}}(P) = 
\exp_{\mathcal{B}} \left( (\mathcal{B} \mathcal{X}^{-1})^{\frac{1}{2}} P (\mathcal{B} \mathcal{X}^{-1})^{\frac{1}{2}} \right),
\end{equation}
Where, \( \mathcal{X}, \mathcal{B} \in \mathcal{S}ym_n^{+} \), with \( \mathcal{X} \) denoting the input data and \( \mathcal{B} \) its Riemannian barycenter, \( P \in T_{\mathcal{X}} \) lies in the tangent space at \( \mathcal{X} \).

Furthermore, \( \mathcal{B} \) (Riemannian barycenter) can also be regarded as the Fréchet mean~\cite{yang2010riemannian}, which minimizes the weighted distance from all data points:
\begin{equation}
\mathcal{B} = \arg\min_{B \in \mathcal{S}ym_{n}^{+}} \sum_{i=1}^{N} w_i \cdot \mathit{d}_{\mathit{p}}^2(B, X_i),
\end{equation}
where \( w_i \geq 0 \) and \( \sum_{i=1}^{N} w_i = 1 \), $\mathit{d}_{\mathit{p}}^2(\cdot, \cdot)$ represents the geodesic distance.

We also introduce a learnable running mean \( \hat{\mathcal{B}} \), which is utilized exclusively during the inference phase to ensure consistency and stability. This running mean is iteratively updated during training using the Riemannian geodesic formula:
\begin{equation}
\hat{\mathcal{B}} = \hat{\mathcal{B}}^{\frac{1}{2}} \left( \hat{\mathcal{B}}^{-\frac{1}{2}} \mathcal{B} \hat{\mathcal{B}}^{-\frac{1}{2}} \right)^w \hat{\mathcal{B}}^{\frac{1}{2}},
\label{eq:kr}
\end{equation}
where \( \hat{\mathcal{B}} \) is initialized as the identity matrix \( I \), \( \mathcal{B} \) denotes the Riemannian barycenter computed from the current batch, and \( w \in (0, 1) \) is a weighting coefficient controlling the update rate.

\section{F Proofs}
\label{proofs}
\subsection{F.1 Proof of Theorem 1}
\label{proof1}
\begin{theorem}
\label{theorem1}
Let \( A \) and \( B \) be two matrices with orthonormal columns, \textit{i.e.}, \( A^\top A = I \) and \( B^\top B = I \). If there exists an invertible matrix \( P \) such that \( B = AP \), then \( A \) and \( B \) span the same subspace, and thus represent the same point on the Grassmannian \( \mathcal{G}(n, p) \). The proof is in Appendix F.1.
\end{theorem}
\begin{proof}
Let \( A \) and \( B \) be two matrices with orthonormal columns, \textit{i.e.},
\begin{equation}
A^\top A = I \quad \text{and} \quad B^\top B = I.
\end{equation}
We aim to show that \( A \) and \( B \) span the same subspace, and hence correspond to the same point on the Grassmannian manifold.

Since \( A \) is orthonormal, its Moore–Penrose pseudoinverse is \( A^\top \). We define a transformation matrix \( P \) as
\begin{equation}
P = A^\top B.
\end{equation}
Then, we can express \( B \) as
\begin{equation}
B = A P.
\end{equation}
Because both \( A \) and \( B \) are orthonormal matrices of the same dimension, \( P \) is square and full-rank, thus invertible.

This implies that the column vectors of \( B \) can be written as linear combinations of the columns of \( A \), and vice versa. Therefore, \( A \) and \( B \) span the same subspace.

According to the definition of the Grassmannian manifold, matrices with column spaces that span the same subspace are considered equivalent, they belong to the same equivalence class and represent the same point on the Grassmannian manifold.

Hence, \( A \) and \( B \) represent the same point on the Grassmannian manifold.
\end{proof}
\subsection{F.2 Proof of Metric Topology Structure on Grassmannian Manifold}
\label{Metric}
\begin{proposition}
\label{metricto}
The Grassmannian manifold is a metric topology.
\end{proposition}
\begin{proof} 
The Grassmannian manifold \( \mathcal{G}(n, p) \) is a metric topology space. We use the Projection metric \( d_p(U, V) \) to measure the distance between two \( p \)-dimensional subspaces \( U \) and \( V \). This metric is computed using the projection distance of the two subspaces in Euclidean space, with the specific expression as shown in  Equation (1). This metric satisfies the following conditions:

\textbf{1.Non-negativity:}
\begin{equation}
d_p(X, Y) = \| P_X - P_Y \|_F \geq 0.
\end{equation}
\textbf{2.Symmetry:}
\begin{equation}
\begin{aligned}
\| P_X - P_Y \|_F = \| P_Y - P_X \|_F \\
\Rightarrow \quad d_p(X, Y) = d_p(Y, X).
\end{aligned}
\end{equation}
\textbf{3.Triangle inequality:}
\begin{equation}
\begin{aligned}
\| P_X - P_Z \|_F \leq \| P_X - P_Y \|_F + \| P_Y - P_Z \|_F \\
\Rightarrow \quad d_p(X, Z) \leq d_p(X, Y) + d_p(Y, Z).
\end{aligned}
\end{equation}
Since this metric satisfies the properties of non-negativity, symmetry, and the triangle inequality, it is a valid metric. Under this metric, the open sets on the Grassmannian manifold are defined by open balls, which are given by the following form:
\begin{equation}
B_d(V, r) = \{ U \in \mathcal{G}(n, p) \mid d(U, V) < r \}.
\end{equation}
The topology generated by these open balls is the metric topology. Therefore, we can conclude that the Grassmannian manifold \( \mathcal{G}(n, p) \) is a metric topological space.
\end{proof}
 \subsection{F.3 Proof of Subspace Convergence on the Grassmannian Manifold}
 \label{Subspace Convergence}
\begin{proposition}[Subspace Convergence on the Grassmannian Manifold]
Let the subspaces generated in each iteration of the training process be \( V_n = \text{span}(A_n) \), where \( A_n \in \mathcal{G}(n, p) \) is a matrix with full column rank, and its columns form an orthogonal basis, satisfying \( A_n^T A_n = I_p \). The converged subspace is denoted as \( V_\infty = \text{span}(A_\infty) \).

To characterize the convergence of \( V_n \to V_\infty \), we introduce the metric structure on the Grassmannian manifold \( \mathcal{G}(n, p) \). One commonly used metric is the projection distance:
\begin{equation}
\begin{aligned}
d_p(V_n, V_\infty) = \left\| P_n - P_\infty \right\|_F,\\
P_n = A_n A_n^T, \quad P_\infty = A_\infty A_\infty^T.
\end{aligned}
\end{equation}
This distance satisfies the four axioms of a metric space and is naturally induced on the Grassmannian manifold.
\end{proposition}

\begin{lemma}[Matrix Convergence \( \Rightarrow \) Subspace Convergence]
If \( \lim_{n \to \infty} \|A_n - A_\infty\|_F = 0 \), then:
\begin{equation}
\lim_{n \to \infty} d_p(V_n, V_\infty) = 0.
\end{equation}
\end{lemma}
\begin{proof}
Since the Frobenius norm converges, the difference between the projection matrices also shrinks:
\begin{equation}
\|P_n - P_\infty\|_F  \leq C \|A_n - A_\infty\|_F .
\end{equation}
where \( C \) is a constant related to \( A_\infty \). Therefore, \( d_p(V_n, V_\infty) \to 0 \), meaning the subspaces converge on the Grassmannian manifold.
\end{proof}
\subsection{F.4 Proof of Convergence Equivalence Under the Logarithmic Map}
\label{Logarithmic}
\begin{proposition}[Convergence Equivalence Under the Logarithmic Map]
To better characterize convergence behavior on the manifold, we use the Logarithmic map to locally linearize the Grassmannian manifold. Let \( V_\infty \in \mathcal{G}(n, p) \) be a reference point, then there exists a logarithmic map:
\begin{equation}
\log_{V_\infty}: \mathcal{G}(n, p) \to T_{V_\infty} \mathcal{G}(n, p),
\end{equation}
which maps the subspace \( V_n \) to a symmetric matrix \( \xi_n \) in the tangent space, \textit{i.e.},
\begin{equation}
\xi_n = \log_{V_\infty}(V_n).
\end{equation}
\textbf{Property:} When \( d_p(V_n, V_\infty) \to 0 \), we have \( \|\xi_n\|_F \to 0 \). This shows that if the subspaces converge on the manifold, their representations in the tangent space also approach the zero vector.
\end{proposition}
\subsection{F.5 Proof of Convergence Equivalence Under the Exponential Map}
\label{Exponential}
\begin{proposition}[The Exponential Map Guarantees Inverse Convergence of Subspaces]
Furthermore, the inverse of the logarithmic map is the exponential map:
\begin{equation}
\exp_{V_\infty}: T_{V_\infty} \mathcal{G}(n, p) \to \mathcal{G}(n, p).
\end{equation}
If there exists a sequence \( \xi_n \in T_{V_\infty} \mathcal{G}(n, p) \) such that \( \|\xi_n\|_F \to 0 \), then:
\begin{equation}
\exp_{V_\infty}(\xi_n) \to V_\infty.
\end{equation}
Thus, if the subspaces converge in the tangent space, it follows that they also converge on the Grassmannian manifold.
\end{proposition}
\subsection{F.6 Proof of Continuity and Subspace Stability in Training}
\label{Continuity}
\begin{corollary}[Continuity and Subspace Stability in Neural Network Training]
The mapping \( A_n = f_\theta^{(n)}(x) \) generated by a neural network during training can be viewed as a continuous map from the input space to the Grassmannian manifold. If the loss function is continuous with respect to \( A_n \), and it converges to a stable state during training, \textit{i.e.},
\begin{equation}
\lim_{n \to \infty} \|A_n - A_\infty\|_F = 0,
\end{equation}
then, based on the propositions above, we have:
\begin{equation}
\lim_{n \to \infty} d_p(\text{span}(A_n), \text{span}(A_\infty)) = 0.
\end{equation}
This implies that the neural network eventually stabilizes at a fixed subspace, completing the learning and convergence of the "optimal subspace."
\end{corollary}

\section{G Implementation details and additional experiments}
\label{experiments}
This section provides additional experimental details on GMSF-Net regarding Grassmannian adaptive multi-subspace construction and interaction.
\subsection{G.1 Datasets and preprocessing}
\label{Datasets}
\textbf{FPHA:} The First-Person Hand Action Benchmark (FPHA) is a skeleton-based hand action recognition dataset comprising 1,175 video clips across 45 distinct hand action categories. Each video records the 3D positions (x, y, z coordinates) of 21 hand joints. During preprocessing, the 3D coordinates of each frame are flattened into a 63-dimensional feature vector. Based on the sequence of frame features, a $63 \times 63$ covariance or correlation matrix is then constructed to represent the entire action sequence in a structured manner. For experimental settings, we follow existing protocols~\cite{wang2021symnet}, using 600 clips for training and 575 for testing, or adopting a subject-based split where subjects 1–3 are used for training and 4–6 for validation. This structured representation facilitates the extraction of stable action features, which benefits subsequent classification tasks.

\textbf{HDM05:} HDM05 is a motion-capture-based action recognition dataset containing 2,273 skeleton sequences performed by various actors. Each frame records the 3D coordinates (x, y, z) of 31 human body joints, allowing each action sequence to be modeled as a $93 \times 93$ covariance matrix. To ensure a balanced class distribution, we follow the preprocessing protocol from~\cite{brooks2019riemannian}, removing under-represented classes and retaining 2,086 sequences across 117 action categories. In our experiments, we adopt the classification setting proposed by Huang and Gool~\cite{huang2018building}, where each video clip is classified into one of the 117 action classes. The covariance-based modelling effectively captures the spatiotemporal relationships among joints and is well-suited for structure-based action recognition tasks.

\textbf{SSVEP:} The MAMEM-SSVEP-II dataset~\cite{pan2022matt} consists of time-synchronized EEG recordings from 11 participants, collected using the EGI 300 Geodesic EEG System (256 channels, with a sampling rate of 250 Hz). In the SSVEP task, participants were asked to focus on one of five distinct visual stimuli flickering at different frequencies (6.66, 7.50, 8.57, 10.00, and 12.00 Hz), each lasting 5 seconds across multiple sessions. In each session, participants performed five trials corresponding to each of the five stimulus frequencies, with each trial divided into four 1-second segments (from 1 to 5 seconds after the cue). Each session consisted of 100 trials in total, ensuring a diverse and rich dataset. For data preprocessing, we follow the data preparation and performance evaluation procedures described in ~\cite{wang2024grassmannian} to ensure a fair comparison and analysis.

\textbf{Disease:} The Disease dataset is generated by simulating the SIR disease propagation model, where node labels indicate infection status and node features represent susceptibility to the disease~\cite{chami2019hyperbolic}. The dataset uses a tree structure to simulate the disease propagation process.

\textbf{PubMed:} The PubMed dataset includes 19,717 nodes, with each node representing a medical research paper. The labels on the nodes correspond to the specific academic subfields of the papers~\cite{sen2008collective}.

\textbf{Cora:} The Cora dataset contains 2,708 nodes, each representing a machine learning research paper. The node labels indicate the academic subfields to which the papers belong~\cite{sen2008collective}.

\textbf{Airport:} The Airport dataset consists of 2236 nodes, where each node represents an airport and edges represent airline routes. Each airport is assigned a label based on the population of the country it is located in, and each airport has a 4-dimensional feature vector containing geographic information~\cite{chami2019hyperbolic}.

\subsection{G.2 Implementation details}
\label{Implementation}
This section mainly introduces the experimental details mentioned above. In the output part of the model, we designed a projection layer and a fully connected layer as the core structure, combined with a Softmax layer to normalize the probability distribution. To train the model, we used hardware equipped with an E5-2699 v4 CPU and 256GB of RAM. The following will provide specific details for Sequential Data and Graph-based Data.
\subsubsection{Sequential Data Learning on Grassmannian Manifolds}
Given a video classification dataset, our goal is to classify each video into one of several categories. To illustrate this, we use the architecture of the PFHA task as an example. Due to the implementation of multiple stacked blocks in the experiment, as described in Optimisation Strategies section, we employ the SPDBN operation~\cite{brooks2019riemannian}, denoted as SPBBN$_{d_i}^{d_{i+1}}$: SPD($d_i$) $\rightarrow$ SPD($d_{i+1}$), which aims to standardize the distribution of manifold data and refine discriminative features. We adopt the model design scheme shown in fig. 2. In this task, the adaptive multi-subspace construction is defined as follows:
\begin{equation}
\begin{aligned}
S_i' = \big[ & \phi(W)_{i_1} \cdot \psi(X)_{i_1},\ \phi(W)_{i_2} \cdot \psi(X)_{i_2},\ \dots, \\
& \phi(W)_{i_p} \cdot \psi(X)_{i_p} \big] \in \mathcal{G}(63, 10),
\end{aligned}
\end{equation}

Where $\psi(X)$ represents the Gram-Schmidt orthogonal function applied to the input feature matrix $X \in \mathbb{R}^{63 \times 63}$ to generate a set of orthogonal subspaces; $\phi(W)$ represents the \textit{Softmax} applied to the initial weights $W \sim \mathcal{N}(0, \sigma^2)$; and $\{i_1, \dots, i_{10}\} = \mathcal{I}_{10}(\phi(W))$ denotes the indices corresponding to the top $10$ largest weights, used to select the most important $10$ subspaces. The interaction between the selected subspaces is then modeled as follows:
\begin{equation}
\Gamma_{new} = \exp_{\Gamma}\left(\Gamma, \alpha \cdot \frac{1}{N} \sum_{i=1}^{N} \log_{\Gamma}(\Gamma, P_{i}) \right),
\end{equation}
where \( \alpha \) is the step size, \( \Gamma \in \mathcal{G}(32, 10)\) is the current mean, and \( P_{i}\in \mathcal{G}(32, 10) \) represents each subspace within the same geometric framework(see GMSR section for details). \( \log(\cdot, \cdot) \) and \( \exp(\cdot, \cdot) \) need to satisfy the intrinsic geometry of the Grassmannian manifold. Finally, we flatten the matrices and apply a linear map from 1024 dimensions to 45 dimensions (representing 45 different actions). We use the aforementioned loss function to train the model.
\subsubsection{Graph-based Learning on Grassmannian Manifolds}
\label{Graph-based}
In all our experiments, graph-structured data is modeled on the Grassmannian manifold. To evaluate the performance of our proposed GMSF-Net on graph datasets, we follow a similar experimental setup to Chami et al.~\cite{chami2019hyperbolic}. Specifically, to reduce the number of parameters, we first apply a linear layer to project the input features to a lower dimension before introducing the AMSE encoder. For link prediction tasks, we use the Fermi-Dirac decoder, while a linear decoder~\cite{chami2019hyperbolic} is employed for node classification tasks.

In experiments involving adaptive subspace construction and fusion, data is projected onto Grassmannian manifolds of specified dimensions. The projection process is implemented via singular value decomposition (SVD), which constructs multiple low-dimensional subspaces that are then mapped onto the Grassmannian manifold. Each subspace is assigned an importance score, initialized by drawing $\omega \sim \mathcal{N}(0, 1)$, which guides the construction of multiple subspaces. The resulting subspaces are then passed to a multi-subspace fusion block, where various geometric mapping frameworks are applied to obtain diverse representations on the Grassmannian manifold. During training, the importance scores $\omega$ are optimised jointly with the rest of the network parameters using a unified optimiser. Further implementation details can be found in the accompanying GitHub repository.

In terms of configuration, we use 25 Grassmannian geometric frameworks for the Disease, Airport, and Cora datasets, and 30 for the Pubmed dataset. The number of training epochs is set to 500 for Disease, 10000 for Airport, and 5000 for both Pubmed and Cora. The Adam optimiser is used, with other hyperparameters (\textit{e.g.}, learning rate and weight decay) determined through random search.
\subsection{G.3 Hyper-parameters experiments}
\label{hyperparameter}
Table 1 presents the impact of subspace dimensionality ($p$) and the number of channels ($C$) under a fixed setting. The ablation studies were conducted on two representative datasets, HDM05 (motion capture) and FPHA (first-person hand action), to evaluate the generality of our design choices. The results demonstrate that both overly small and excessively large configurations hinder performance: the former limits representational capacity, while the latter introduces redundancy or noise. Moderate configurations (e.g., $p{=}10$ or $C{=}10/20$) consistently achieve optimal or near-optimal results across both datasets, highlighting the critical role of appropriate hyperparameter selection. Moreover, increasing either subspace dimensionality or channel number beyond a certain point yields diminishing returns, reflecting a saturation effect. All ablation experiments were conducted with one interaction block, underscoring the trade-off between sensitivity and robustness within this structural constraint.
\small
\setlength{\tabcolsep}{2pt}
\begin{table*}[ht]
  \label{tab}
  \centering
  \begin{tabular}{lcccccccc}
    \toprule
    \multicolumn{1}{c}{} & \multicolumn{4}{c}{\textbf{Subspace Dimensionality ($p$)}} & \multicolumn{4}{c}{\textbf{Number of Channels ($C$)}} \\
    \cmidrule(lr){2-5} \cmidrule(lr){6-9}  
    \textbf{Dataset} & \textbf{5} & \textbf{10} & \textbf{20} & \textbf{30} & \textbf{5} & \textbf{10} & \textbf{20} & \textbf{30} \\
    \midrule 
    HDM05   & 59.99\%±2.08 & \textbf{63.64\%}±1.24 & 63.43\%±1.46 & 53.84\%±1.16 & 61.44\%±0.71 & \textbf{63.64\%}±1.24 & 64.32\%±0.85 & 62.84\%±1.07 \\
    FPHA    & 85.27\%±2.27 & \textbf{90.43\%}±0.74 & 88.90\%±0.56 & 76.67\%±1.09 & 87.47\%±1.49 & 90.43\%±0.74 & 89.95\%±1.35 & \textbf{90.73\%}±1.62 \\
    \bottomrule
  \end{tabular}
  \caption{Hyperparameter Analysis: Subspace Dimensionality and Channel Number.}
\end{table*}

Table 2 presents an ablation study under the two interaction blocks (2 Blocks) setting, examining the impact of subspace dimensionality $p$ and channel architecture configuration in the adaptive subspace construction (AdMSC) process. Performance was evaluated across four channel configurations ({5,5,5}, {5,5,10}, {5,10,10}, and {5,10,20}) as $p$ increased from 5 to 20. The results demonstrate that balancing subspace dimensionality and channel number significantly affects model performance: increasing $p$ steadily improves results (\textit{e.g.}, from 85.04\% to 89.36\% in {5,5,5}), confirming that higher-dimensional subspaces better capture the Grassmannian manifold in multi-interaction architectures. Meanwhile, increasing channels notably boosts performance at medium dimensions ($p=10$), with {5,10,10} achieving peak accuracy of 90.70\% at $p=20$. However, excessive channels ({5,10,20}) cause performance decline (88.11\%) due to redundancy or training instability, showing diminishing returns. This study highlights the necessity of an optimal balance between subspace dimensionality and channel width for effective AdMSC.

\small
\setlength{\tabcolsep}{2pt}
\begin{table}[ht]
  \label{Hyperparameter}
  \centering
  \begin{tabular}{lccc}
    \toprule
    \multicolumn{1}{c}{} & \multicolumn{3}{c}{\textbf{Subspace Dimensionality ($p$)}} \\
    \cmidrule(lr){2-4} 
    \textbf{Channel} & \textbf{5} & \textbf{10} & \textbf{20} \\
    \midrule 
    \{5,5,5\}     & 85.04\%±2.61 & 87.57\%±2.06 & 89.36\%±0.73 \\
    \{5,5,10\}    & 84.96\%±2.26 & 86.28\%±2.22 & 89.04\%±0.96 \\
    \{5,10,10\}   & 85.72\%±2.87 & 87.17\%±1.42 & \textbf{90.70\%}±0.70 \\
    \{5,10,20\}   & \textbf{87.27\%}±3.22 & \textbf{88.43\%}±2.22 & 88.11\%±1.22 \\
    \bottomrule
  \end{tabular}
  \caption{Performance under Varying Subspace Dimensionality and Channel Settings with 2 Blocks.}
\end{table}


Table 3 presents the performance of GMSF-Net on the FPHA and HDM05 datasets under different subspace quantity settings (2, 5, 7, and 10) in the AdMSC, with and without the interaction mechanism. The primary purpose of this ablation study is to analyze the impact trend of a key hyperparameter in GMSF-Net, namely the number of subspaces in AdMSC, on the overall model performance.

The AdMSC demonstrates strong robustness to its subspace quantity hyperparameter, as evidenced by limited performance fluctuations (\textit{e.g.}, a difference of less than 1.46\% in the FPHA dataset with interaction) when varying subspaces from 2 to 10, with or without interaction mechanisms. While expanding subspaces may yield marginal gains (\textit{e.g.}, 10-subspace configurations), the non-linear relationship between computational and memory overhead, and performance enhancement necessitates strategic balancing, particularly since redundant subspaces may offer identical performance to moderate configurations (\textit{e.g.}, 7 or 10 subspaces). These findings advocate selecting mid-range subspace counts (5–7) to optimise the efficiency-effectiveness equilibrium, reserving higher configurations for scenarios where incremental gains justify substantial resource investments.
\begin{table*}[htbp]
  \label{AdMSC}
  \centering
  \begin{tabular}{lcccccccc}
    \toprule
    \multicolumn{1}{c}{} & \multicolumn{4}{c}{\textbf{Without Interaction}} & \multicolumn{4}{c}{\textbf{With Interaction}} \\
    \cmidrule(lr){2-5} \cmidrule(lr){6-9}  
    \textbf{Dataset} & \textbf{2} & \textbf{5} & \textbf{7} & \textbf{10} & \textbf{2} & \textbf{5} & \textbf{7} & \textbf{10} \\
    \midrule 
    FPHA   & 78.13\%±3.25 & 80.68\%±2.15 & 81.04\%±1.41  & \textbf{81.41\%}±2.03 & 88.81\%±0.95 & \textbf{90.43\%}±0.74 & 89.68\%±0.80 & 90.27\%±0.91 \\
    HDM05  & 52.91\%±1.36 & 56.49\%±1.37 & 56.85\%±1.60 & \textbf{58.87\%}±1.34 & \textbf{63.98\%}±1.56 & 63.64\%±1.24 & 63.24\%±1.84 & 63.36\%±1.10 \\
    \bottomrule
  \end{tabular}
  \caption{Ablation Study on the Number of Subspaces in AdMSC.}
\end{table*}

\subsection{G.4 Feature Visualization of GMSF-Net on SSVEP}
\label{visualization}
\cref{figure5} and~\cref{figure6} present the visualization results generated by GDLNet~\cite{wang2024grassmannian}. It is important to note that the EEG signals of subject11 are primarily concentrated in the Oz channel. Therefore, in the visualization analysis, a strong gradient response in the Oz channel is expected, serving as an indicator of whether the model successfully captures the key information. Compared with the visualizations of GMSF-Net shown in~\cref{figure3} and~\cref{figure4}, both models exhibit gradient responses concentrated in the Oz channel, clearly demonstrating the effectiveness of Riemannian methods in processing EEG signals. However, GMSF-Net shows a more focused and pronounced response within the 0.25 to 0.75 seconds, which aligns with prior studies on the relationship between SSVEP signals and the Oz channel in EEG research~\cite{han2018highly,herrmann2001human}, whereas GDLNet displays a relatively dispersed channel-level response. This finding further validates the superiority of the proposed Grassmannian multi-subspace interaction mechanism in extracting essential geometric features over conventional Grassmannian-based models.

The above analysis is based on the spatial distributions corresponding to different visual stimuli at various frequencies in the MAMEM-SSVEP-II dataset. The topographic maps under all frequency conditions exhibit similar distribution patterns, consistently showing strong gradient responses in the Oz channel across all frequencies. Moreover, despite variations in visual stimulation frequency, the spatial location of the gradient responses on the scalp remains remarkably stable throughout the entire time period, further reflecting the model’s sensitivity to stable geometric structures.

In summary, GMSF-Net can effectively identify and capture subtle discrepancies hidden within overall similar spatial distributions, particularly in the dynamic responses of the Oz channel under five frequency conditions. This capability highlights the model’s sensitivity and precision in decoding non-stationary brain activity from SSVEP-EEG signals. The above visualization results provide strong evidence for the superiority and robustness of the proposed GMSF-Net in modelling dynamic EEG signals.
\begin{figure*}[htbp]
\centering
\includegraphics[width=\linewidth,trim={5 20 3 20},clip=true]{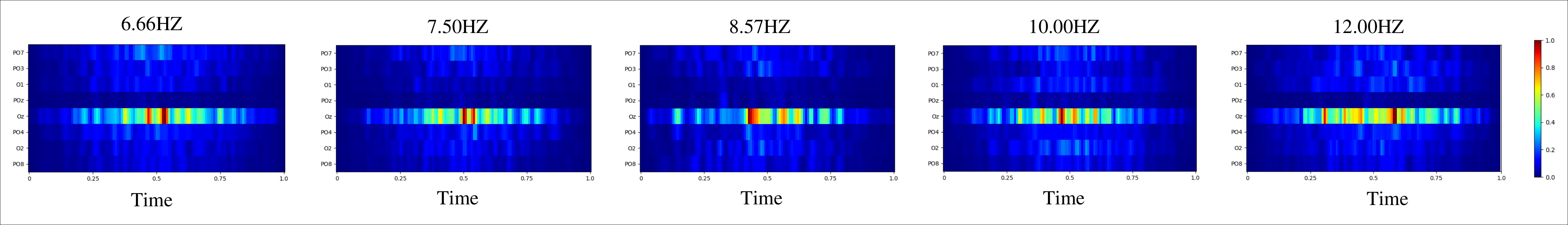}
\caption{Based on GMSF-Net, the heatmaps of absolute gradient responses (computed following the method in~\cite{pan2022matt}) are presented across five frequency categories on the MAMEM-SSVEP-II dataset (S11). In each heatmap, the x-axis represents time, and the y-axis corresponds to EEG channels.}
\label{figure3}
\end{figure*}

\begin{figure*}[htbp]
\centering
\includegraphics[width=\linewidth,trim={5 20 3 20},clip=true]{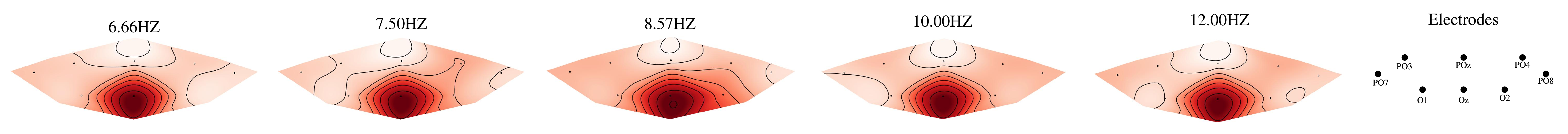}
\caption{The spatial topomaps of the mean absolute gradient responses across time for the S11 subject on the MAMEM-SSVEP-II dataset, based on GMSF-Net (computed as in~\cite{pan2022matt}. The brain region marked in dark red corresponds to channel Oz, which exhibits strong gradient activation across the visual cortex under all stimulation frequencies.
}
\label{figure4}
\end{figure*}

\begin{figure*}[htbp]
\centering
\includegraphics[width=\linewidth,trim={5 20 3 20},clip=true]{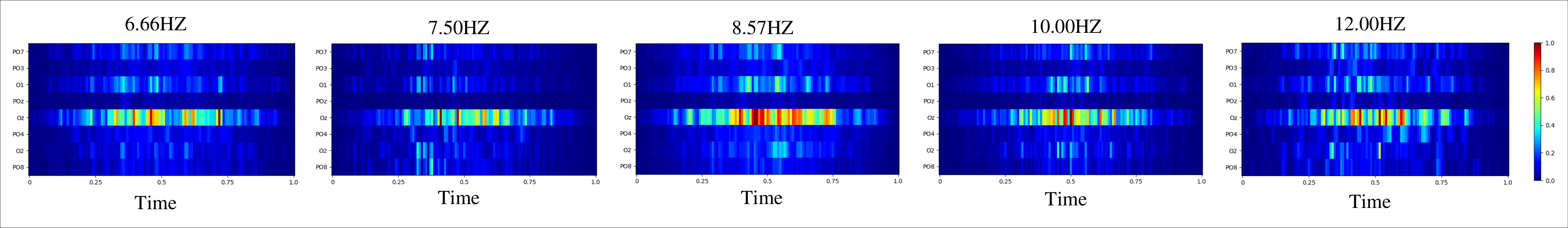}
\caption{The heatmaps of five frequency classes of the MAMEM-SSVEP-II dataset demonstrated by GDLNet.}
\label{figure5}
\end{figure*}

\begin{figure*}[htbp]
\centering
\includegraphics[width=\linewidth,trim={5 20 3 20},clip=true]{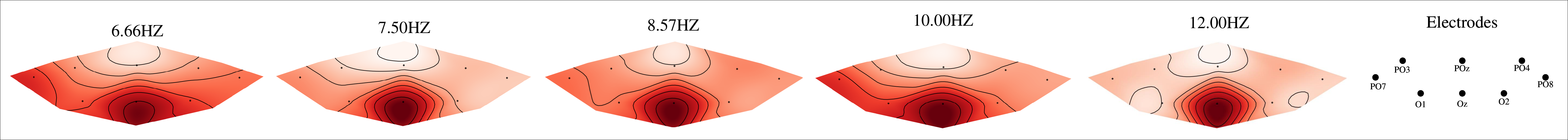}
\caption{The brain topology maps obtained on the MAMEM-SSVEP-II dataset demonstrated by GDLNet.}
\label{figure6}
\end{figure*}

\end{document}